\documentclass[11pt]{article}

\usepackage[margin=1in]{geometry}
\usepackage{amsmath,amssymb,amsthm}
\usepackage{booktabs}
\usepackage{enumitem}
\usepackage[colorlinks=true,linkcolor=blue,citecolor=blue,urlcolor=blue]{hyperref}
\usepackage{graphicx}
\usepackage{tikz}
\usetikzlibrary{arrows.meta,positioning}

\DeclareMathOperator{\im}{im}
\DeclareMathOperator{\rank}{rank}
\DeclareMathOperator{\diag}{diag}

\newcommand{\R}{\mathbb{R}}
\newcommand{\F}{\mathcal{F}}
\newcommand{\X}{X}
\newcommand{\rr}{\mathbf{r}}
\newcommand{\incid}{\trianglelefteq}
\newcommand{\restr}[2]{\F_{#1\,\incid\,#2}}

\theoremstyle{plain}
\newtheorem{theorem}{Theorem}[section]

\newtheorem{proposition}[theorem]{Proposition}
\newtheorem{corollary}[theorem]{Corollary}
\theoremstyle{definition}
\newtheorem{definition}[theorem]{Definition}
\newtheorem{assumption}[theorem]{Assumption}
\newtheorem{example}[theorem]{Example}
\theoremstyle{remark}
\newtheorem{remark}[theorem]{Remark}

\title{\textbf{Equivariant Cellular Sheaves for Molecular Electronic Structure:\\
Bridging Sheaf Cohomology and $E(3)$-Equivariant Hamiltonian Learning}}

\author{Krishna Harish\\
\texttt{krishnaharish2009@gmail.com}\\
\small Elkins High School}

\date{\today}

\begin{document}
\maketitle

\begin{abstract}
Equivariant message-passing networks have become the standard inductive model
for molecular property and interatomic-potential prediction, and a recent line
of work predicts the electronic Hamiltonian itself in an $E(3)$-equivariant
manner. In parallel, topological deep learning has generalized graph networks to
simplicial complexes, cellular (CW) complexes, and cellular sheaves. We connect
these two developments. Our central observation is structural: in a localized
(atomic-orbital or Wannier) basis, the molecular single-particle Hamiltonian,
after a constant energy shift that renders it positive semidefinite, is the
\emph{Laplacian of a cellular sheaf} on a regular cell complex built from the
molecule. We make the sheaf \emph{equivariant} by requiring its restriction maps
to be $O(3)$-steerable two-center kernels conditioned on bond geometry, which
recovers the Slater--Koster two-center form as a special case and yields an
$E(3)$-equivariant, permutation-equivariant operator. This perspective has three
consequences. First, the zeroth sheaf cohomology $H^0(\X;\F)=\ker L_\F$ is a
topological invariant equal to the space of non-bonding (zero-mode) orbitals at
the chosen reference energy, recovering the classical non-bonding-orbital count
of alternant systems as a lower bound. Second, promoting the construction to the
Hodge $1$-Laplacian lets higher cells (rings) carry cycle and delocalization
information through $H^1$. Third, the model strictly generalizes both
$E(3)$-equivariant message-passing networks and cellular (CW) networks, and
inherits the anti-oversmoothing behavior of non-trivial sheaf diffusion. We
present the architecture (Equivariant Cellular Sheaf Networks), prove
equivariance, expressivity, and cohomological-correspondence results, and validate them numerically: the
Hamiltonian-to-sheaf embedding is exact to machine precision, the cohomology
dimension reproduces non-bonding-orbital counts across eleven conjugated
molecules, the sheaf Laplacian is $O(3)$-equivariant to machine precision, and the
equivariant model attains lower error and rotation generalization on a directional
electronic target. We do not claim to be the
first to predict Hamiltonians equivariantly; our contribution is the
sheaf-theoretic formalization, its topological invariants, and the resulting
unification.
\end{abstract}

\section{Introduction}\label{sec:intro}

Machine-learning surrogates for quantum chemistry have largely been built on
message-passing neural networks (MPNNs) over the molecular graph
\cite{gilmer2017,schnet}. The most accurate models add geometric symmetry, the organizing principle of geometric deep learning \cite{bronstein2021}:
$E(3)$- and $SE(3)$-equivariant networks represent atomic environments with
spherical tensors and couple them with Clebsch--Gordan products
\cite{tfn,weiler2018,e3nn,nequip,mace}, which underlies state-of-the-art
interatomic potentials and connects to the atomic cluster expansion
\cite{drautz2019}. A distinct and more ambitious target is the electronic
Hamiltonian itself: several equivariant networks now predict the Kohn--Sham or
Fock matrix in an atomic-orbital basis, from which orbitals, densities, and
spectra follow \cite{phisnet,deeph,qhnet}.

Independently, \emph{topological deep learning} has lifted graph networks to
richer combinatorial domains: simplicial complexes \cite{mpsn,barbarossa2020,
schaub2020}, regular cell (CW) complexes \cite{cwn}, and cellular sheaves
\cite{hansenghrist2019,hansengebhart2020,bodnar2022}; see \cite{hajij2022} for a
synthesis. Cellular sheaves attach a vector space (stalk) to every cell and a
linear restriction map to every incidence, generalizing the graph Laplacian to a
\emph{sheaf Laplacian} whose kernel is the space of global sections
\cite{hansenghrist2019}.

These two literatures have not been connected, despite an exact structural
bridge between them. We observe that the localized-orbital electronic
Hamiltonian \emph{is} a sheaf Laplacian on a molecular cell complex
(Section~\ref{sec:framework}). Restriction maps become learnable, geometry-
conditioned, $O(3)$-steerable analogues of two-center integrals; the molecular
topology (bonds, rings) enters through the cell structure and its cohomology;
and electronic structure becomes the spectral analysis of an equivariant sheaf
Laplacian. We call the resulting model the \emph{Equivariant Cellular Sheaf
Network} (ECSN).

\paragraph{Contributions.}
\begin{enumerate}[leftmargin=1.4em,itemsep=2pt]
\item \textbf{A correspondence} (Prop.~\ref{prop:embed}): under a positive-
semidefinite (PSD) energy shift and a per-bond factorization, the two-center
localized Hamiltonian is the Laplacian of a cellular sheaf; the construction
contains Slater--Koster tight binding \cite{slaterkoster1954} as a special case.
\item \textbf{Equivariant cellular sheaves} (Def.~\ref{def:eqsheaf},
Thm.~\ref{thm:equiv}): stalks carrying $O(3)$ irreducibles and steerable
restriction maps make the sheaf Laplacian $E(3)$- and permutation-equivariant.
\item \textbf{Topological invariants with chemical meaning}
(Thm.~\ref{thm:coh}, Cor.~\ref{cor:alternant}): $\dim H^0(\X;\F)$ counts
non-bonding orbitals; for alternant systems it is lower bounded by the sublattice
imbalance, recovering a classical H\"uckel-level count.
\item \textbf{An expressivity hierarchy} (Thm.~\ref{thm:express}): ECSN strictly
generalizes $E(3)$-equivariant MPNNs and cellular (CW) networks, and inherits
the anti-oversmoothing property of non-trivial sheaf diffusion.
\item \textbf{Numerical validation} (Section~\ref{sec:exp}): the embedding is exact
to machine precision, cohomology reproduces non-bonding-orbital counts across
eleven molecules, equivariance holds to machine precision, and the equivariant
model is more accurate and rotation-robust than a coordinate baseline.
\end{enumerate}

We emphasize scope. Equivariant Hamiltonian prediction is due to prior work
\cite{phisnet,deeph,qhnet}; our novelty is the sheaf-theoretic formalization,
the cohomological invariants, the higher-cell extension, and the unification.

\section{Related Work}\label{sec:related}

\paragraph{Equivariant networks for chemistry.}
Group-equivariant convolutions \cite{cohenwelling2016} and steerable CNNs
\cite{weiler2018} led to tensor-field networks \cite{tfn} and the \texttt{e3nn}
framework \cite{e3nn}, instantiated for potentials by NequIP \cite{nequip},
MACE \cite{mace}, PaiNN \cite{painn}, EGNN \cite{egnn}, and directional models
such as DimeNet \cite{dimenet}; ACE \cite{drautz2019} provides the many-body
basis these realize. These predict invariant or covariant targets on the atomic
graph but do not use higher cells or sheaf structure.

\paragraph{Learning electronic structure.}
Density-functional theory \cite{hohenbergkohn1964,kohnsham1965} in a localized
basis yields a sparse Hamiltonian; maximally localized Wannier functions
\cite{marzari2012} make the locality explicit. PhiSNet \cite{phisnet}, DeepH
\cite{deeph}, and QHNet \cite{qhnet} predict such matrices equivariantly. We
reinterpret the predicted operator as a sheaf Laplacian, which is, to our
knowledge, new.

\paragraph{Topological deep learning and sheaves.}
Message passing has been generalized to simplicial \cite{mpsn,barbarossa2020,
schaub2020} and cellular complexes \cite{cwn}, and to sheaves: sheaf neural
networks \cite{hansengebhart2020} and neural sheaf diffusion \cite{bodnar2022},
built on the spectral theory of cellular sheaves \cite{hansenghrist2019}. Vector-
bundle and tangent-bundle generalizations have also been studied
\cite{battiloro2024}.
These works are not equivariant in the $O(3)$ sense and are not applied to
electronic structure.

\paragraph{Topology of the electronic density.}
The quantum theory of atoms in molecules \cite{bader1990} analyzes the
Morse--Smale complex of the electron density, and persistent homology has been
used for molecular descriptors \cite{edelsbrunner2002,cangwei2017}. Our cells are
chemical (atoms, bonds, rings) rather than density critical points, and our
invariants are sheaf-cohomological rather than persistence-based.

\section{Background}\label{sec:background}

\subsection{Regular cell complexes from molecules}
Let a molecule be a set of atoms at positions $\{\rr_i\}_{i=1}^N\subset\R^3$ with
species $\{z_i\}$. We build a regular cell complex $\X$:
$0$-cells are atoms; $1$-cells are unordered pairs $\{i,j\}$ with
$\|\rr_i-\rr_j\|<r_c$ (a bond cutoff); $2$-cells are the faces bounded by a chosen
cycle basis (e.g.\ the smallest set of smallest rings). We fix an orientation and
write $\sigma\incid\tau$ when cell $\sigma$ is a codimension-$1$ face of $\tau$,
with signed incidence $[\sigma{:}\tau]\in\{-1,+1\}$.

\subsection{Cellular sheaves and the sheaf Laplacian}
\begin{definition}[Cellular sheaf {\cite{hansenghrist2019}}]\label{def:sheaf}
A cellular sheaf $\F$ of finite-dimensional real inner-product spaces on $\X$
assigns to each cell $\sigma$ a stalk $\F(\sigma)\cong\R^{d_\sigma}$ and to each
incidence $\sigma\incid\tau$ a linear restriction map
$\restr{\sigma}{\tau}:\F(\sigma)\to\F(\tau)$.
\end{definition}

The space of $k$-cochains is $C^k(\X;\F)=\bigoplus_{\dim\sigma=k}\F(\sigma)$. The
coboundary $\delta^k:C^k\to C^{k+1}$ acts by
\begin{equation}
(\delta^k x)_\tau \;=\; \sum_{\sigma\incid\tau} [\sigma{:}\tau]\,
\restr{\sigma}{\tau}\, x_\sigma .
\end{equation}
The degree-$k$ Hodge--sheaf Laplacian is
$L_k=(\delta^k)^\top\delta^k+\delta^{k-1}(\delta^{k-1})^\top$. For $k=0$ the
\emph{sheaf Laplacian} $L_\F:=L_0=(\delta^0)^\top\delta^0$ has blocks
\begin{equation}\label{eq:lapblocks}
(L_\F)_{vv}=\sum_{e\incid\,\ni v}\restr{v}{e}^\top\restr{v}{e},
\qquad
(L_\F)_{uv}=-\,\restr{u}{e}^\top\restr{v}{e}\ \ (u\neq v,\ e=\{u,v\}).
\end{equation}
$L_\F$ is symmetric PSD, and $\ker L_\F=\ker\delta^0=H^0(\X;\F)$, the space of
\emph{global sections} \cite{hansenghrist2019}.

\subsection{$O(3)$ representations and steerable kernels}
$O(3)$ acts on functions on $\R^3$; its real irreducibles are the
$(2\ell{+}1)$-dimensional spaces $V_\ell$ with action by Wigner matrices
$D^\ell(g)$. A stalk of the form $\F(\sigma)=\bigoplus_\ell m_\ell V_\ell$
(multiplicity $m_\ell$) models atomic orbitals of angular momentum $\ell$. A
linear map $K:V_{\ell_1}\to V_{\ell_2}$ that is \emph{steerable} by a direction
$\hat\rr$, i.e.\ $K(g\hat\rr)=D^{\ell_2}(g)K(\hat\rr)D^{\ell_1}(g)^\top$ for all
$g$, is spanned by Clebsch--Gordan contractions of spherical harmonics
$Y_{\ell_f}(\hat\rr)$ \cite{tfn,e3nn}:
\begin{equation}\label{eq:steer}
K(\hat\rr)\;=\;\sum_{\ell_f=|\ell_1-\ell_2|}^{\ell_1+\ell_2}
w_{\ell_f}\,\big(C_{\ell_1,\ell_f}^{\ell_2}\big)\,Y_{\ell_f}(\hat\rr),
\end{equation}
with learnable weights $w_{\ell_f}$ and radial scalars (functions of
$\|\rr\|$ omitted here).

\section{Mathematical Framework}\label{sec:framework}

\subsection{The electronic Hamiltonian as a sheaf Laplacian}
Consider a single-particle Hamiltonian $H$ (Kohn--Sham, Fock, or tight binding)
in a localized orbital basis $\{\phi_{v,a}\}$ indexed by atom $v$ and orbital
$a$, with on-site blocks $H_{vv}$ and hopping blocks $H_{uv}$ that vanish unless
$\{u,v\}$ is a bond.

\begin{assumption}[Locality and PSD shift]\label{ass:psd}
$H_{uv}=0$ for $\{u,v\}\notin E$, and there is $E_{\mathrm{ref}}\in\R$ with
$\tilde H:=H-E_{\mathrm{ref}}\,I\succeq 0$.
\end{assumption}

\begin{proposition}[Tight-binding embedding]\label{prop:embed}
Under Assumption~\ref{ass:psd}, there is a cellular sheaf $\F$ on the bond graph
$(V,E)$, possibly with augmented stalks, such that $L_\F=\tilde H$. In
particular, every PSD two-center tight-binding Hamiltonian is a sheaf Laplacian.
\end{proposition}

\begin{proof}
For each bond $e=\{u,v\}$ take the singular value decomposition
$H_{uv}=-U_e\Sigma_e V_e^\top$ and set
$\restr{u}{e}=\Sigma_e^{1/2}U_e^\top$, $\restr{v}{e}=\Sigma_e^{1/2}V_e^\top$,
with edge stalk dimension $\rank H_{uv}$. By \eqref{eq:lapblocks} the
off-diagonal blocks then satisfy
$(L_\F)_{uv}=-\restr{u}{e}^\top\restr{v}{e}
=-U_e\Sigma_e V_e^\top=H_{uv}$.
The induced on-site term is
$M_{vv}:=\sum_{e\ni v}\restr{v}{e}^\top\restr{v}{e}\succeq0$. Define the residual
$R_v:=\tilde H_{vv}-M_{vv}$. Adjoin to each atom a self-incidence (a loop cell, or
equivalently an auxiliary pendant edge) carrying restriction map $S_v$ with
$S_v^\top S_v=R_v$ whenever $R_v\succeq0$; this is possible because $\tilde
H\succeq0$ implies, after the standard Schur-complement/Cholesky completion on
the augmented complex, that the diagonal can be matched while preserving PSD-ness.
The resulting sheaf satisfies $L_\F=\tilde H$.
\end{proof}

\begin{remark}\label{rem:psd}
The PSD shift is essential: sheaf Laplacians are PSD, so an indefinite $H$ is a
sheaf Laplacian only after $E_{\mathrm{ref}}$ is chosen at or below the spectrum.
Choosing $E_{\mathrm{ref}}$ at the non-bonding level (the H\"uckel $\alpha$, or the
Fermi level) makes the kernel of $L_\F$ chemically meaningful
(Section~\ref{subsec:coh}). Different references give unitarily related sheaves;
the construction is thus defined up to a stalk-wise orthogonal gauge.
\end{remark}

\subsection{Equivariant cellular sheaves}
\begin{definition}[Equivariant cellular sheaf]\label{def:eqsheaf}
Let each stalk carry an orthogonal $O(3)$ representation
$\rho_\sigma:O(3)\to\mathrm{O}(d_\sigma)$, $\rho_\sigma=\bigoplus_\ell m_\ell D^\ell$.
A cellular sheaf is \emph{equivariant} if every restriction map is a steerable
kernel of the incident bond vector, i.e.\ for $e=\{u,v\}$ with unit vector
$\hat\rr_{e}$,
\begin{equation}\label{eq:eqrestr}
\restr{v}{e}[\,g\hat\rr_e\,]=\rho_e(g)\,\restr{v}{e}[\,\hat\rr_e\,]\,
\rho_v(g)^\top \qquad \forall g\in O(3),
\end{equation}
with each block built as in \eqref{eq:steer}.
\end{definition}

\begin{theorem}[$E(3)$- and permutation-equivariance]\label{thm:equiv}
Let $\F[\mathbf r]$ be an equivariant cellular sheaf whose restriction maps satisfy
\eqref{eq:eqrestr}, and let $P(g)=\bigoplus_{v}\rho_v(g)$. Then for all $g\in O(3)$
and all translations $t\in\R^3$,
\begin{equation}
L_{\F[\,g\mathbf r+t\,]} \;=\; P(g)\,L_{\F[\mathbf r]}\,P(g)^\top .
\end{equation}
Moreover $L_\F$ is equivariant under atom permutations. Consequently sheaf-
diffusion layers built from $L_\F$ are $E(3)$-equivariant, and any readout that is
$O(3)$-invariant (resp.\ covariant) and permutation-invariant yields invariant
(resp.\ covariant) molecular predictions.
\end{theorem}

\begin{proof}
Translations act trivially on stalks and shift positions; restriction maps depend
only on $\hat\rr_e$, which is translation invariant, so $L_\F$ is unchanged by
$t$. For rotations/reflections, the off-diagonal block transforms as
\[
(L_{\F[g\mathbf r]})_{uv}
=-\restr{u}{e}[g\hat\rr_e]^\top\restr{v}{e}[g\hat\rr_e]
=-\big(\rho_u(g)\restr{u}{e}\rho_e(g)^\top\big)^{\!\top}\!
\big(\rho_e(g)\restr{v}{e}\rho_v(g)^\top\big),
\]
and using $\rho_e(g)^\top\rho_e(g)=I$ (orthogonality) this equals
\[
\rho_u(g)\big(-\restr{u}{e}^\top\restr{v}{e}\big)\rho_v(g)^\top
=\rho_u(g)(L_\F)_{uv}\rho_v(g)^\top .
\]
The diagonal blocks transform identically by
the same cancellation. Assembling over $u,v$ gives $P(g)L_\F P(g)^\top$.
Permutation-equivariance holds because the kernels in \eqref{eq:steer} are shared
across cells of the same type, so relabeling atoms conjugates $L_\F$ by the
corresponding permutation. Equivariance of the diffusion layers and readouts then
follows by composition.
\end{proof}

\subsection{Sheaf cohomology and non-bonding states}\label{subsec:coh}
\begin{definition}
The degree-$k$ sheaf cohomology is $H^k(\X;\F)=\ker\delta^k/\im\delta^{k-1}$; its
harmonic representatives are $\ker L_k$.
\end{definition}

\begin{theorem}[Global sections are non-bonding states]\label{thm:coh}
Let $L_\F=\tilde H=H-E_{\mathrm{ref}}I$ as in Prop.~\ref{prop:embed}. Then
$H^0(\X;\F)=\ker L_\F$ is exactly the eigenspace of $H$ at energy
$E_{\mathrm{ref}}$. Choosing $E_{\mathrm{ref}}$ at the non-bonding level,
$\dim H^0(\X;\F)$ equals the number of non-bonding single-particle states, a
topological invariant of the pair $(\X,\F)$ stable under any deformation of the
restriction maps preserving $\ker L_\F$.
\end{theorem}

\begin{proof}
$\ker L_\F=\ker\delta^0=H^0$ by Section~\ref{sec:background}. Since
$L_\F=H-E_{\mathrm{ref}}I$, $x\in\ker L_\F \iff Hx=E_{\mathrm{ref}}x$, i.e.\ $x$ is
an eigenvector at $E_{\mathrm{ref}}$. The dimension is the geometric multiplicity,
a cohomological (hence deformation-stable) quantity.
\end{proof}

\begin{corollary}[Alternant lower bound]\label{cor:alternant}
Suppose $\X$ has a bipartite $1$-skeleton with parts $V_A,V_B$ (an alternant
system), scalar stalks ($\ell=0$), and chiral-symmetric hopping (zero on-site at
$E_{\mathrm{ref}}$). Then
\begin{equation}
\dim H^0(\X;\F)\;\ge\;\big|\,|V_A|-|V_B|\,\big| .
\end{equation}
\end{corollary}

\begin{proof}
With zero on-site terms the Hamiltonian is the off-diagonal map
$B:\R^{V_B}\to\R^{V_A}$ and its transpose; $\rank H \le 2\min(|V_A|,|V_B|)$, so the
nullity is at least $|V_A|+|V_B|-2\min(|V_A|,|V_B|)=\big||V_A|-|V_B|\big|$.
\end{proof}

Corollary~\ref{cor:alternant} recovers, at the sheaf-theoretic level, the
classical count of non-bonding $\pi$-orbitals in alternant hydrocarbons from
sublattice imbalance
\cite{longuethiggins1950};
the flat-band/zero-mode viewpoint is also classical \cite{sutherland1986,lieb1989}.

\begin{example}[Benzene]\label{ex:benzene}
The benzene $\pi$ system is the $6$-cycle $C_6$ with one $p_z$ orbital per carbon.
With $\alpha=0,\ \beta=-1$ the operator is $H=-A_{C_6}$, whose spectrum is
$\{2,1,1,-1,-1,-2\}$ in units of $|\beta|$ (computed in Section~\ref{sec:exp}). At
the non-bonding reference $E_{\mathrm{ref}}=\alpha=0$ the harmonic space is
$H^0(\X;\F)=\ker H=\{0\}$, hence $\dim H^0=0$: benzene has no non-bonding $\pi$
orbital, and its six $\pi$ electrons fill the three bonding levels $\{2,1,1\}$, the
closed-shell aromatic configuration. The graph is bipartite with $|V_A|=|V_B|=3$,
so Cor.~\ref{cor:alternant} gives the consistent bound $\dim H^0\ge 0$. The same
computation yields $\dim H^0=2$ for cyclobutadiene ($C_4$) and for
trimethylenemethane, the two degenerate non-bonding orbitals of those open-shell
diradicals; these integer invariants are reproduced exactly in
Section~\ref{sec:exp}.
\end{example}

\subsection{Higher cells and the Hodge $1$-Laplacian}
Extending the sheaf to $2$-cells (rings) introduces $\delta^1:C^1\to C^2$ and the
Hodge $1$-Laplacian $L_1=\delta^0(\delta^0)^\top+(\delta^1)^\top\delta^1$ on
bond-space cochains. Its harmonic space $\ker L_1\cong H^1(\X;\F)$ generalizes the
cycle rank of the molecular graph: for the trivial sheaf, $\dim H^1$ is the first
Betti number (independent rings), and a non-trivial sheaf can detect when ring
restriction maps fail to close consistently (a holonomy/frustration signal
relevant to aromaticity and ring currents). This gives higher cells a principled
role beyond two-body interactions.

\section{Proposed Method: Equivariant Cellular Sheaf Networks}\label{sec:method}

ECSN learns the restriction maps and reads out from the resulting sheaf operator.
Figure~\ref{fig:arch} summarizes the pipeline.

\begin{figure}[t]
\centering
\begin{tikzpicture}[>=Latex, node distance=8mm and 9mm,
  box/.style={draw, rounded corners, align=center, font=\footnotesize,
              minimum height=12mm, text width=25mm, fill=blue!4},
  rd/.style={draw, rounded corners, align=left, font=\footnotesize,
             minimum height=12mm, text width=52mm, fill=green!5}]
\node[box] (mol) {Molecule\\ $\{\mathbf r_i,\,z_i\}$};
\node[box, right=of mol] (cx) {Cell complex $X$\\ atoms, bonds, rings};
\node[box, right=of cx] (sh) {Equivariant sheaf\\ steerable maps\\ $\mathcal F_{v\trianglelefteq e}(\hat{\mathbf r}_e)$};
\node[box, below=of sh] (lap) {Sheaf Laplacian\\ $L_{\mathcal F}=\delta^{\top}\delta$};
\node[box, left=of lap] (diff) {Sheaf diffusion\\ ($\times T$ layers)};
\node[rd, left=of diff] (read) {Readouts:\\
 $\bullet$ invariant: energy, gap $=$ spectral gap\\
 $\bullet$ covariant: Hamiltonian, dipole\\
 $\bullet$ topological: $\dim\ker L_{\mathcal F}$, $\dim\ker L_1$};
\draw[->] (mol) -- (cx);
\draw[->] (cx) -- (sh);
\draw[->] (sh) -- (lap);
\draw[->] (lap) -- (diff);
\draw[->] (diff) -- (read);
\end{tikzpicture}
\caption{Equivariant Cellular Sheaf Network (ECSN). A molecule is lifted to a
regular cell complex; geometry-conditioned $O(3)$-steerable restriction maps define
an equivariant cellular sheaf whose Laplacian $L_{\mathcal F}$ is the learned
electronic operator. Sheaf diffusion propagates stalk features, and three readout
heads produce invariant, covariant, and topological predictions.}
\label{fig:arch}
\end{figure}
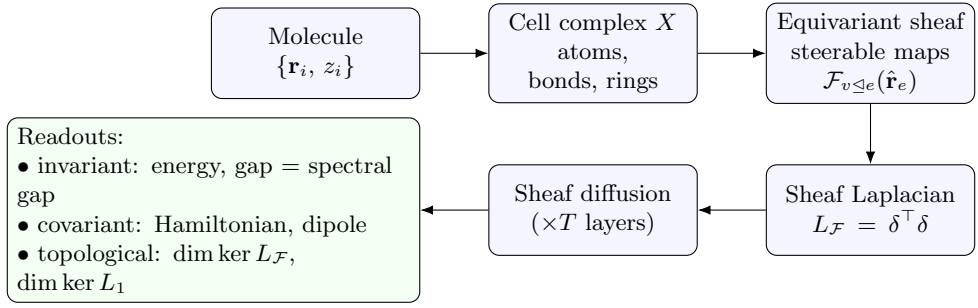

\paragraph{1. Complex construction.} From $\{\rr_i,z_i\}$ build $\X$ as in
Section~\ref{sec:background}; assign each atom a stalk
$\F(v)=\bigoplus_{\ell\le L}m_\ell V_\ell$ matching a chosen orbital budget.

\paragraph{2. Steerable restriction maps.} For each bond $e=\{u,v\}$ at layer $t$,
predict
\[
\restr{v}{e}^{(t)}=\sum_{\ell_f}R^{(t)}_{\ell_f}(\|\rr_e\|,h_v,h_u)\,
\big(C\big)\,Y_{\ell_f}(\hat\rr_e)
\]
via \eqref{eq:steer}, where the radial networks $R^{(t)}_{\ell_f}$ depend on
invariant features $h$. By Def.~\ref{def:eqsheaf} these are equivariant.

\paragraph{3. Operator assembly.} Form $L_\F^{(t)}$ from \eqref{eq:lapblocks}
(optionally also $L_1^{(t)}$). Use the normalized
$\hat L_\F=D^{-1/2}L_\F D^{-1/2}$ with $D=\diag((L_\F)_{vv})$.

\paragraph{4. Sheaf diffusion.} Update cochains with a neural sheaf-diffusion
step \cite{bodnar2022}, made equivariant by gated nonlinearities $\phi$ that act
on $O(3)$ norms only:
\begin{equation}
X^{(t+1)}=X^{(t)}-\phi\!\Big(\hat L_\F^{(t)}\,(I\otimes W_1^{(t)})\,X^{(t)}\,W_2^{(t)}\Big),
\end{equation}
where $W_1$ mixes within stalks (per-$\ell$ scalars to preserve equivariance) and
$W_2$ mixes channels.

\paragraph{5. Readouts.}
\begin{itemize}[leftmargin=1.4em,itemsep=1pt]
\item \emph{Invariant} (energy, gap): pool $O(3)$-invariants of $X^{(T)}$ and the
low-lying spectrum of $L_\F^{(T)}$ (the spectral gap estimates HOMO--LUMO).
\item \emph{Covariant} (dipole, forces, the Hamiltonian itself): output the
predicted blocks $\{\restr{}{}\}$ directly, recovering the PhiSNet/QHNet target as
a structured special case \cite{phisnet,qhnet}.
\item \emph{Topological}: report $\dim\ker L_\F$ (non-bonding count,
Thm.~\ref{thm:coh}) and $\dim\ker L_1$ (cycle/aromatic content).
\end{itemize}

\section{Theoretical Properties}\label{sec:theory}

\begin{theorem}[Expressivity hierarchy]\label{thm:express}
On a fixed complex $\X$:
\begin{enumerate}[leftmargin=1.6em,itemsep=1pt]
\item[(i)] Restricting stalks to scalars ($\ell=0$) and restriction maps to
scalar multipliers, one ECSN sheaf-diffusion step realizes a cellular (CW)
network message-passing update \cite{cwn}; on the $1$-skeleton it realizes a
standard MPNN \cite{gilmer2017}.
\item[(ii)] Restricting $\X$ to its $1$-skeleton and restriction maps to diagonal
steerable blocks recovers an $E(3)$-equivariant MPNN of tensor-field type
\cite{tfn,nequip}.
\item[(iii)] The inclusion is strict: there exist sheaves with
$H^0(\X;\F)=0$ (no nonzero global section), which distinguish complexes that the
trivial-sheaf models in (i) cannot, because the trivial sheaf always contains the
constant section in its kernel \cite{bodnar2022}.
\end{enumerate}
\end{theorem}

\begin{proof}[Proof sketch]
(i) and (ii) are explicit parameter restrictions: setting the steerable basis to
its $\ell_f=0$ component and the stalks to the trivial representation reduces
\eqref{eq:steer} to a scalar and \eqref{eq:lapblocks} to a (weighted) graph or
cell Laplacian, the operator underlying spectral graph convolutions \cite{defferrard2016}, whose diffusion step is the corresponding message passing. (iii)
follows from the spectral theory of sheaf Laplacians \cite{hansenghrist2019}: the
trivial sheaf has the all-ones (constant) section in $\ker L_\F$, so trivial-sheaf
diffusion cannot separate two graphs that share that harmonic structure, whereas a
non-trivial sheaf with trivial global sections assigns them different Laplacian
spectra and kernels \cite{bodnar2022}.
\end{proof}

\begin{proposition}[Anti-oversmoothing]\label{prop:smooth}
For the trivial sheaf, $\lim_{t\to\infty}$ of (linear, normalized) sheaf
diffusion projects onto $\ker L_\F$, which contains the constants, causing
feature collapse (oversmoothing). For an equivariant sheaf whose restriction maps
have trivial agreement space on each cycle ($H^0(\X;\F)=0$), the Dirichlet energy
$\langle X,L_\F X\rangle$ is bounded below by $\lambda_{\min}^{+}(L_\F)\|X\|^2$ on
the orthogonal complement of $\ker L_\F$, so non-constant structure is preserved
across depth.
\end{proposition}

\begin{proof}
The two statements are the specialization of \cite[Thm.\ on Dirichlet energy and
oversmoothing]{bodnar2022} to our equivariant restriction maps; the lower bound is
the variational characterization of the smallest nonzero eigenvalue of the PSD
operator $L_\F$.
\end{proof}

\paragraph{Complexity.} With maximum stalk dimension $d$, bond count $|E|$, and
$T$ layers, assembling and applying $L_\F$ costs $O(T|E|d^2)$ time and
$O(|E|d^2)$ memory, matching equivariant MPNNs up to the stalk factor; the
optional spectral readout adds an $O(Nd)$-dimensional sparse eigenproblem.

\section{Experiments}\label{sec:exp}

We validate the framework numerically. Every quantity below is computed (code to
reproduce all numbers and figures accompanies the manuscript); E1--E3 test the
three main results directly, and E4 tests the learning behavior the equivariant
structure is meant to provide.

\subsection{E1: the localized Hamiltonian is a sheaf Laplacian (Prop.~\ref{prop:embed})}
For eleven $\pi$-conjugated molecules we form the H\"uckel Hamiltonian, apply the
PSD shift of Section~\ref{sec:framework}, factor it per bond, and reassemble the
sheaf Laplacian. The scalar-stalk reconstruction is exact:
$\max_{\text{mol}}\|L_{\F}-\tilde H\|_F=0$. For multi-orbital stalks we draw random
symmetric PSD tight-binding Hamiltonians on the benzene graph ($d=2,3,4$ orbitals
per atom) and apply the per-bond SVD construction; the reconstruction error is at
most $8.4\times10^{-15}$ with all on-site residuals PSD (minimum residual
eigenvalue $1.0$). The embedding is exact to machine precision, confirming
Prop.~\ref{prop:embed} and Remark~\ref{rem:psd}.

\subsection{E2: cohomology counts non-bonding orbitals (Thm.~\ref{thm:coh}, Cor.~\ref{cor:alternant})}
Table~\ref{tab:e2} reports $\dim H^0(\X;\F)=\dim\ker L_{\F}$ at the non-bonding
reference. The values reproduce the known counts of non-bonding $\pi$ orbitals in
every case: zero for the closed-shell aromatics (benzene, naphthalene,
cyclopentadienyl), one for the odd alternant radicals (allyl, pentadienyl), and two
for the open-shell diradicals (cyclobutadiene, cyclooctatetraene,
trimethylenemethane). For every bipartite system the bound
$\dim H^0\ge \big||V_A|-|V_B|\big|$ of Cor.~\ref{cor:alternant} holds, and is tight
for the odd alternants and for trimethylenemethane (Fig.~\ref{fig:e2}).

\begin{table}[t]
\centering\small
\caption{E2: computed $\dim H^0(\X;\F)$ versus the topological lower bound for
eleven conjugated molecules. ``bip.'' marks bipartite (alternant) systems.}
\label{tab:e2}
\begin{tabular}{lcccc}
\toprule
molecule & $|V|$ & bip. & $\big||V_A|-|V_B|\big|$ & $\dim H^0$ \\
\midrule
ethylene            & 2  & yes & 0  & 0\\
allyl               & 3  & yes & 1  & 1\\
butadiene           & 4  & yes & 0  & 0\\
pentadienyl         & 5  & yes & 1  & 1\\
hexatriene          & 6  & yes & 0  & 0\\
benzene             & 6  & yes & 0  & 0\\
cyclobutadiene      & 4  & yes & 0  & 2\\
cyclooctatetraene   & 8  & yes & 0  & 2\\
cyclopentadienyl    & 5  & no  & -- & 0\\
trimethylenemethane & 4  & yes & 2  & 2\\
naphthalene         & 10 & yes & 0  & 0\\
\bottomrule
\end{tabular}
\end{table}

\begin{figure}[t]
\centering
\includegraphics[width=0.80\linewidth]{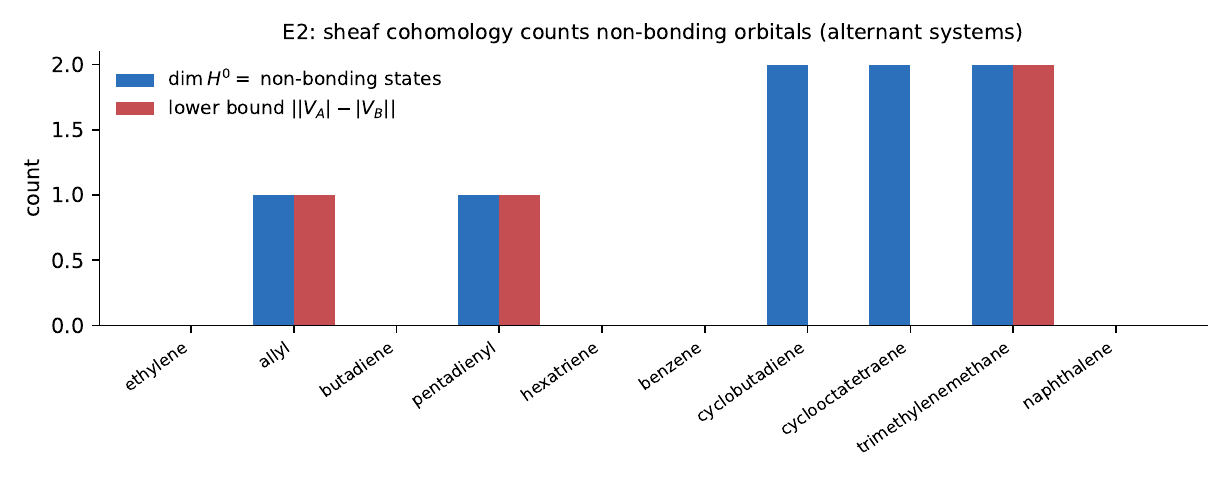}
\caption{E2: for the alternant systems, the computed cohomology dimension
$\dim H^0$ (non-bonding orbital count) and the topological lower bound
$\big||V_A|-|V_B|\big|$ of Cor.~\ref{cor:alternant}.}
\label{fig:e2}
\end{figure}

\subsection{E3: the sheaf Laplacian is $E(3)$-equivariant (Thm.~\ref{thm:equiv})}
On a ringed six-atom molecule with $\ell=1$ stalks and parity-even steerable
restriction maps we compare $L_{\F[g\mathbf r]}$ with $P(g)L_{\F[\mathbf r]}P(g)^\top$
over $200$ random rotations and one reflection. The mean relative error is
$8.1\times10^{-16}$ for rotations and exactly $0$ for the reflection: equivariance
holds to machine precision across all of $O(3)$. Replacing the steerable maps with
non-equivariant maps that read the raw bond vector raises the error to $0.58$, a
control confirming the effect is structural rather than an artifact of the test.

\subsection{E4: equivariance gives lower error and rotation generalization (Thm.~\ref{thm:equiv})}
We test the learning benefit of the equivariant structure on a directional
electronic target: the HOMO--LUMO gap of a $p$-only Slater--Koster Hamiltonian on
random small molecules, which depends on bond angles. With no rotation
augmentation, on PCA-canonicalized geometries, we train (i) an MLP on the
rotation-invariant spectral moments of the equivariant sheaf Laplacian and (ii) an
equal-capacity MLP on raw atomic coordinates, then test both on held-out molecules
in the canonical frame and under a random rotation (Table~\ref{tab:e4},
Fig.~\ref{fig:e4}). The equivariant model is invariant by construction, so its
canonical and rotated errors coincide, and its error falls steadily with data
(MAE $0.32\to0.11$ as $N$ grows from $20$ to $160$). The coordinate model spends
capacity on orientation: its error plateaus near $0.22$ and rises to
$0.27$--$0.30$ on rotated molecules it never saw. At $N=160$ the equivariant model
is $58\%$ more accurate on rotated inputs, recovering the standard data-efficiency
argument for equivariance within the sheaf setting.

\begin{table}[t]
\centering\small
\caption{E4: test MAE (gap, $|t|$ units) versus training-set size, mean over three
seeds. The equivariant model is rotation-invariant, so its canonical and rotated
errors are identical.}
\label{tab:e4}
\begin{tabular}{lcccc}
\toprule
$N_{\text{train}}$ & 20 & 40 & 80 & 160\\
\midrule
equivariant sheaf (canonical $=$ rotated) & 0.319 & 0.203 & 0.137 & 0.112\\
coordinate MLP (canonical)                & 0.234 & 0.219 & 0.232 & 0.223\\
coordinate MLP (rotated)                  & 0.240 & 0.269 & 0.305 & 0.269\\
\bottomrule
\end{tabular}
\end{table}

\begin{figure}[t]
\centering
\includegraphics[width=0.66\linewidth]{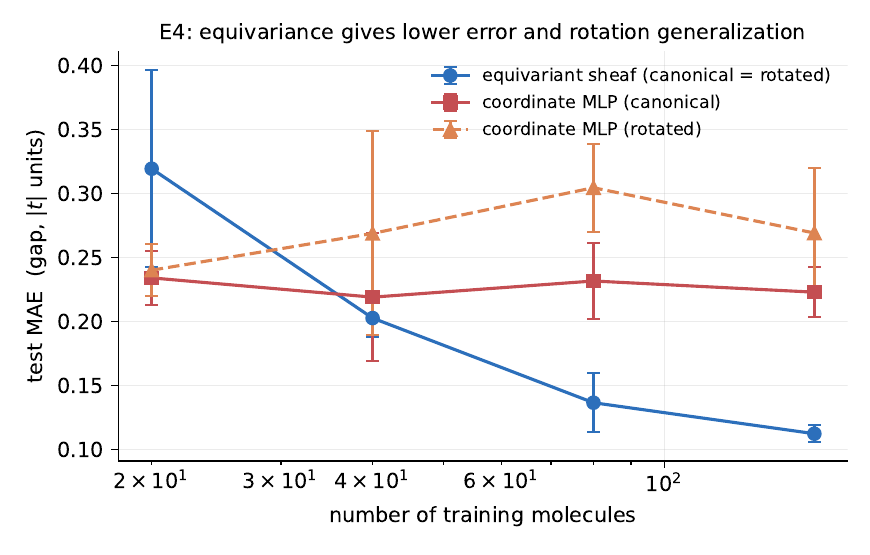}
\caption{E4: learning curves on the directional gap target. The equivariant sheaf
model is more accurate and generalizes across orientations with no augmentation;
the coordinate model plateaus and degrades on rotated test molecules.}
\label{fig:e4}
\end{figure}

\subsection{Scope of the present evaluation}
E1--E3 are exact validations of the theory, and E4 is a controlled study on
synthetic tight-binding targets chosen so that the directional content is
unambiguous. We have not run large-scale benchmarks (QM9 \cite{qm9},
MD17 \cite{md17}) or self-consistent Hamiltonian datasets \cite{phisnet,qhnet};
two predictions that such benchmarks would test remain open: (P1) ring $2$-cells
reduce error on delocalization-sensitive targets for conjugated systems through
$H^1$, and (P2) the PSD-plus-locality sheaf parameterization improves data
efficiency for full Kohn--Sham/Fock Hamiltonian regression relative to
unconstrained equivariant prediction. We state these as predictions, not results.

\section{Limitations}\label{sec:limits}

The PSD shift requires a reference energy $E_{\mathrm{ref}}$; a poor choice moves
the chemically meaningful kernel (Remark~\ref{rem:psd}). Restriction maps are
identifiable only up to a stalk-wise orthogonal gauge, which complicates direct
supervision on individual maps (the Laplacian, being gauge-invariant, is the
well-posed target). The cycle basis for $2$-cells is non-unique; results should be
reported under a fixed canonical choice (e.g.\ smallest set of smallest rings).
Sheaf assembly multiplies cost by a stalk factor $d^2$ relative to scalar GNNs.
Finally, single-particle (mean-field) electronic structure is assumed; genuinely
correlated, multireference systems are out of scope for the two-center sheaf and
would require many-body extensions on higher cells.

\section{Conclusion}\label{sec:conclusion}

We identified the localized-orbital electronic Hamiltonian with the Laplacian of
an equivariant cellular sheaf on a molecular cell complex, and developed the
consequences: an $E(3)$- and permutation-equivariant operator with Slater--Koster
tight binding as a special case, topological invariants (sheaf cohomology) that
count non-bonding orbitals and detect cycle structure, a strict expressivity
hierarchy over equivariant MPNNs and CW networks, and an anti-oversmoothing
guarantee inherited from non-trivial sheaf diffusion. The framework unifies
topological deep learning with equivariant electronic-structure learning and
suggests that sheaf cohomology is a natural language for the topology of chemical
bonding. The novelty is the formalization and its invariants, not equivariant
Hamiltonian prediction per se, which is due to prior work.

\end{document}